\documentclass[11pt]{article}

\usepackage[letterpaper, left=1in, right=1in, top=1in,bottom=1in]{geometry}

\usepackage{parskip}

\usepackage{booktabs} % For formal tables

\usepackage[utf8]{inputenc}

\usepackage{geometry}

\usepackage{graphpap,amscd,mathrsfs,graphicx,lscape,enumitem,dsfont,bm,url,subfigure}
\usepackage{epsfig,amstext,xspace}
\usepackage{algorithm,comment}
\usepackage{algorithmic}

\usepackage{auxiliary}
\usepackage{thmtools,thm-restate}

\usepackage{algorithm,algorithmic}
\usepackage[utf8]{inputenc} % allow utf-8 input
\usepackage[T1]{fontenc}    % use 8-bit T1 fonts
\usepackage{hyperref}       % hyperlinks
\usepackage{url}            % simple URL typesetting
\usepackage{booktabs}       % professional-quality tables
\usepackage{amsfonts}       % blackboard math symbols
\usepackage{nicefrac}       % compact symbols for 1/2, etc.
\usepackage{microtype}      % microtypography

\usepackage{color}              % Need the color package
\usepackage{color-edits}
\addauthor{tl}{cyan}
\addauthor{ab}{brown}

\begin{document}
\title{Advancing subgroup fairness via sleeping experts}

 \author{
  Avrim Blum\thanks{Toyota Technological Institute at Chicago, \texttt{avrim@ttic.edu}. The author was supported in part by NSF grants CCF-1815011 and CCF-1733556.}
  \and 
 Thodoris Lykouris\thanks{Microsoft Research, \texttt{thlykour@microsoft.com}. Research initiated during the author's visit to TTI-Chicago while he was a Ph.D. student at Cornell University. The author was supported in part under NSF grant CCF-1563714 and a Google Ph.D fellowship.} 
 }
\date{December 2019}
\maketitle

%\begin{abstract}
%\input{abstract}
%\end{abstract}

%\addtocounter{page}{-1}
%\thispagestyle{empty}
%\newpage

% ------------------------------------------- BELOW IS THE EC VERSION -------------------------------------------
% \newpage

\begin{abstract}
    % !TEX root = main.tex
We study methods for improving fairness to subgroups in settings with overlapping populations and sequential predictions. Classical notions of fairness focus on the balance of some property across different populations. However, in many applications the goal of the different groups is not to be predicted equally but rather to be predicted well. We demonstrate that the task of satisfying this guarantee for multiple overlapping groups is not straightforward and show that for the simple objective of unweighted average of false negative and false positive rate, satisfying this for overlapping populations can be statistically impossible even when we are provided predictors that perform well separately on each subgroup. On the positive side, we show that when individuals are equally important to the different groups they belong to, this goal is achievable; to do so, we draw a connection to the sleeping experts literature in online learning. Motivated by the one-sided feedback in natural settings of interest, we extend our results to such a feedback model.  We also provide a game-theoretic interpretation of our results, examining the incentives of participants to join the system and to provide the system full information about predictors they may possess.  We end with several interesting open problems concerning the strength of guarantees that can be achieved in a computationally efficient manner.
\end{abstract}

\section{Introduction}
\label{sec:intro}
% !TEX root = main.tex
Concerns about ethical use of data in algorithmic decision-making have spawned an important conversation regarding machine learning techniques that are fair towards the affected populations. We focus here on binary decision-making: e.g., deciding whether to approve a loan, admit a student to an honors class, display a particular job advertisement, or prescribe a particular drug. While multiple fairness notions have been suggested to inform these decisions, most assume access to labeled data and that this data is drawn from i.i.d. distributions. In practice, data patterns are often dynamically evolving and the feedback received is biased by the decisions of the algorithms --- such as only learning whether a student should have been admitted to an honors class if the student is actually admitted --- which induces additional misrepresentation of the actual data patterns. Despite the rich literature, approaching fairness considerations without these strong assumptions is rather underexplored.

Most fairness notions impose a requirement of balance across groups.  For example, demographic parity \cite{Calders09buildingclassifiers} aims to ensure that the percentage of positive predictions is the same across different populations, and equality of opportunity \cite{hardt2016equality} aims to ensure that the percentage of false negative predictions is the same across them. These notions are useful in identifying inequities between subpopulations, especially in settings such as criminal recividism \cite{kleinberg2017human} where there is a conflict between the incentives of the participants and the goal of accurate prediction. However, these notions may not be appropriate when the goal of each group is to be predicted as accurately as possible, and explicitly performing worse on one group in order to produce balance would be morally objectionable or absurd.  For example, in a health application, a balance notion may lead to penalizing the majority population by willfully providing it worse treatment to make amends for the fact that a minority population is not classified correctly due to insufficient data. This would be clearly inappropriate.

More generally, there are many natural scenarios where the goal of each group is just to be predicted as accurately as possible. A student wishes to be admitted in an honors class only if they are qualified to succeed in it; otherwise their academic record may be jeopardized. A person requesting a microloan can be often significantly harmed by receiving it unless they return it (see \cite{LiuDeRoSiHa18} for an interesting discussion). A drug prescription is only beneficial if it helps enhance the health of the patient; otherwise it may cause adverse effects. In these settings what groups care about is not being treated \emph{equally} but rather being treated \emph{well}.

In this work, we consider a fairness notion for sequential non-i.i.d. settings where the subpopulations and the designer both strive for accurate predictions. We say that a prediction rule $f$ is unfair to a subgroup $g$ with respect to a family of rules ${\cal F}$ if there is some other rule $f_g \in {\cal F}$ that performs significantly better than $f$ on $g$, in which case we say that $f_g$ witnesses this unfairness.  More generally, given a collection of rules, some of which may come from the designer, some from third-party entities, and some from the groups themselves, our goal is to achieve performance on each group comparable to that of the best of these rules for that group, even when groups overlap.  Moreover, we aim to achieve this goal with as strong bounds as possible in a challenging ``apple-tasting'' feedback model where we only receive feedback on positive predictions (e.g., when a loan is given or a student is admitted). Interestingly, while our main results are positive, we show such guarantees are {\em not} possible if we replace the performance measure of accuracy (or error) with any fixed weighted average of false-positive and false-negative rates. Our positive results can also be thought of as a form of {\em individual rationality} with respect to the groups: no group has any incentive (up to low order terms) to pull out and, say, form its own lending agency just for members of that group.  From this perspective, we also consider notions of \emph{incentive compatibility} (could groups have any incentive to hide prediction rules from the system) and present a computationally-inefficient algorithm along with an open problem related to achieving this guarantee in a computationally-efficient manner.

\subsection{Our contribution}
We consider a decision-maker with some global decision function
$f_o$ that she would like to use (say to decide who gets a loan), and a collection of groups $\mathcal{G}$, where each group $g \in \mathcal{G}$ proposes some function $f_g$ that {\em it} would like to be used instead on members of $g$.\footnote{Both decision-maker and groups may have more functions; the guarantees are with respond to the best of them.} The guarantee we aim to give is that our overall performance will be nearly as good (or better) than $f_o$ on the entire population with respect to the objective function of the decision-maker, and for each group $g$ our performance is nearly as good (or better) than the performance of $f_g$ with respect to group $g$'s objective. We would like to do this even when groups are overlapping and even when feedback on whether or not a decision was correct is only received when the decision made was ``yes'' (e.g., when the loan was given or the student was admitted).  

Surprisingly, we show that when groups are overlapping and when the group objectives are to minimize the {\em unweighted} average of false-positive rate (FPR) and false-negative rate (FNR), there exist settings where {\em every} global prediction rule $f$ must be unfair to one of the groups. In particular, we present a simple example with two overlapping groups having predictors $f_1$ and $f_2$ respectively, where performing nearly as well as $f_1$ on group 1 and nearly as well as $f_2$ on group 2 is fundamentally impossible when performance is measured as (FPR$+$FNR)$/2$ (or as max(FPR,FNR) or as any fixed non-degenerate weighting) even when the input does not arrive in an online manner. This shows that just having the fairness notions be the same across groups and aligned with the goals of the designer (who in this case does not have any additional goal other than to eliminate unfairness) is not by itself sufficient to be able to achieve our fairness criteria under objectives based on fixed combinations of FPR and FNR.

\begin{informal}[Theorem~\ref{thm:unweighted_average_example}]
\vspace{0.1in}
If subgroups are overlapping and their objectives are to minimize the unweighted average of false positive and false negative rates, there exist instances where no global function can simultaneously perform nearly as well on each group as the best function for that group.  This holds even in the batch setting.
\end{informal}

Instead we aim for low absolute error on each subgroup\footnote{This is equivalent to FPR on a group weighted by the fraction of negative examples in that group plus FNR on a group is weighted by the fraction of positive examples in that group.} and show a connection of this notion to an adversarial online learning setting, that of \emph{sleeping experts} \cite{blum1997empirical,freund1997using}. In sleeping experts, each predictor (also referred to as an expert) can decide at each round to either make a prediction (fire) or abstain (sleep). Sleeping experts algorithms guarantee that, for any expert, the performance of the algorithm when the expert fires is nearly as good as the performance of this expert. Providing the functions $f_o$ for the decision-maker and $f_g$ for each group into existing sleeping-experts algorithms (viewing $f_g$ as abstaining on any individual outside of group $g$) yields the following.

\begin{informal}[Theorem~\ref{thm:subgroup_via_sleeping}]
\vspace{0.1in}
For the objective of minimizing absolute error, we can perform nearly as well as $f_g$ for all groups $g$ while performing nearly as well as $f_o$ overall.
\end{informal}

One particular complication that arises in many fairness-related settings, however, is that feedback received is one-sided: we only learn about the outcome if the loan is given, the student is admitted in the class, the advertisement is displayed, or the drug is prescribed; we do not learn about what would have happened when the action is not taken. In online learning, this is known as the \emph{apple tasting} model \cite{HelmboldLL92_apples}. We therefore  initiate the study of sleeping experts in this apple tasting feedback model, aiming to achieve as strong regret guarantees as possible on a per-group basis.
Combining apple tasting with sleeping experts poses interesting challenges as the exploration needs to be carefully coordinated among different subpopulations. In Section \ref{sec:sleeping_apples}, we provide three different black-box reductions with different advantages in their performance guarantee. 

\begin{informal}[Theorems~\ref{thm:first_reduction}, \ref{thm:second_reduction}, and \ref{thm:third_reduction}]
\vspace{0.1in}
Even if we only receive one-sided feedback, we can still perform nearly as well as $f_g$ for all groups $g$ while performing nearly as well as $f_o$ overall.
\end{informal}
Each of our guarantees is somewhat suboptimal. Theorem~\ref{thm:first_reduction} is based on a construction that does not use sleeping experts, but has an exponential dependence on the number of groups which makes it computationally inefficient when there are many  groups. Theorem~\ref{thm:second_reduction} is a natural adaptation of sleeping experts to this setting but has an error bound with a (sublinear) dependence on the size of the total population instead of only depending on size of the subgroup, which makes the result less meaningful for small groups as this term may dominate their regret bound. Last, Theorem~\ref{thm:third_reduction} has a more involved use of sleeping experts and does not suffer from the two previous issues, but has a suboptimal dependence on the size of the subgroup. Combining the advantages of these approaches without the resulting shortcomings is an intriguing open question.

The final contribution of our work (Section~\ref{sec:incentives}) is to provide a game-theoretic investigation of our setting in terms of incentives of the participating groups. In mechanism design, two important properties that a mechanism should satisfy are Individual Rationality (IR) and Incentive Compatibility (IC).  The former asks that no player should prefer to opt out and seek service outside of the system instead. The latter refers to the mechanism creating no incentives for players to misreport their private information. Inspired by the kidney exchange literature \cite{RothSonmezUnver07,AshlagiRoth2014free,AshlagiFisKasPro13}, we consider each group as a player in our system. The IR property is satisfied when group $g$ can get no benefit (asymptotically) from being predicted by their individual predictor $f_g$, say via their own loan agency. This is exactly what our above guarantees provide and therefore they can be interpreted as asymptotically IR. This observation brings up the question of whether incentive compatibility is also satisfied by sleeping experts algorithms. In this context, IC means that if a group $g$ has a set of predictors $\{f_g\}$ then they get no benefit (asymptotically) from hiding some of those predictors from the decision-maker. Unfortunately, we show that current sleeping experts algorithms are {\em not} (even approximately) incentive compatible. On the other hand, we provide an algorithm that achieves both IR and IC guarantees, as well as operating in the apple tasting setting, at the expense of being computationally inefficient (enumerating over all exponentially-many group intersections). This leads to an interesting open question of finding a computationally efficient algorithm that satisfies both IR and IC properties.

\begin{informal}[Theorem~\ref{thm:negative_IC}]
\vspace{0.1in}
Classical sleeping experts algorithms such as AdaNormalHedge do not satisfy the IC property. 
\end{informal}

\begin{informal}[Theorem~\ref{thm:positive_IC}]
\vspace{0.1in}
Separate multiplicative weights algorithms for each intersection of groups satisfy the IC property at the expense of being computationally inefficient.
\end{informal}

\begin{open}[Section~\ref{ssec:open_IC}]
\vspace{0.1in}
Does there exist a computationally efficient algorithm satisfying both IR and IC properties?
\end{open}

\subsection{Related work}
\label{ssec:related_work}
% !TEX root = main.tex
There is a growing literature aiming to identify natural fairness notions and understand the limitations they impose; see \cite{DworkHPRZ2012awareness, hardt2016equality, Kleinberg2017InherentTI, chouldechova2017fair, kilbertus2017avoiding, agarwal2018reductions} for a non-exhaustive list. With respect to fairness among different demographic groups, notions such as disparate impact \cite{Calders09buildingclassifiers,FeldmanFMSV2015disparate} or equalized odds \cite{hardt2016equality} aim to achieve some balance in performance across different populations. These make sense when there is an intrinsic conflict between the desires of the different groups and those of the designer and can help identify undesirable inequities in the system that require remedies (that are often non-algorithmic and need policy changes).
Unfortunately, aiming to satisfy these notions can also have undesired implications such as intentionally misclassifying some groups to make amends for the inability to classify other groups well enough. This has given rise to an important debate around alternative notions that do not suffer of these issues (see for example \cite{corbett2018measure}). Our work aims to advance this direction.

When populations are overlapping, there are several recent works in the batch setting that tackle considerations similar to the ones we address. Kearns et al. \cite{kearns2017gerrymandering} provide a simple example illustrating the issue that one can be non-discriminatory with respect to, say, gender and race in isolation but heavily discriminate against, say, black female participants; they also discuss how to audit whether a classifier exhibits such unfairness with respect to groups. In a similar motivation, Hebert-Johnson et al. \cite{pmlr-v80-hebert-johnson18a} et al. suggest a fairness notion they term mutli-calibration that relates 
to accurate prediction of all populations that are \emph{computationally identifiable}. Both these works show that their settings are equivalent to agnostic learning, which has strong computational hardness results but tends to have good heuristic guarantees.\footnote{Their connection implies that for exact auditing, a linear dependence on the number of subpopulations is required but can be overcome if additional structure exists. Our work also has a linear dependence on the number of subgroups as it explicitly lists the subgroups. Understanding if this can be improved in our setting when the subgroups and predictors have additional structure} is an interesting open direction. Subgroup fairness is also discussed by Kim et al.~\cite{KimGhorZou2019,KimReinRoth2018} for the related notion of multi-accuracy, with the aim of post-processing data to achieve this notion while maintaining overall accuracy \cite{KimGhorZou2019} as well as metric subgroup fairness notions \cite{KimReinRoth2018}. Our approach
similarly considers overlapping demographic groups, but focuses on the more complex setting of online decision-making under limited feedback and addresses inherent conflicts between the incentives of different subgroups. In fact, even for the batch setting, our impossibility result of Section \ref{sec:incompatibility_example} with respect to the criterion of unweighted average of false positive and false negative rates is of a purely statistical flavor (has nothing to do with computational issues or arrivals being online). It therefore provides insights on complications that exist even when all the incentives seem well aligned. {On the other hand, our connection to sleeping experts does not appear to directly have implications to multicalibration and multiaccuracy for similar reasons as the issues encountered regarding Incentive Compatibility; these notions tend to require a similar \emph{no negative regret} property in order to be satisfied in an online context (similar to our results in Sections~\ref{ssec:negative_IC}-\ref{ssec:positive_IC}).}

Fairness issues in online decision-making have also recently been considered. Joseph et al. \cite{JosephKMR2016fairness} focus on a bandit learning approach and impose what they call \emph{meritocratic fairness} against individuals with some stochastic but unknown quality. Celis et al. \cite{celis2018algorithmic} discuss how to alleviate bias in settings like online advertising where bandit learning can lead to overexposure or underexpsure to particular actions. Blum et al. \cite{BlumGuLySr18} focus on adversarial online learning and examine for different fairness notions whether non-discriminatory predictors can be combined efficiently in an online fashion while preserving non-discrimination. A recent line of work also focuses on online settings where decisions made today affect what is allowed tomorrow as they need to be connected via some metric-based notion of individual fairness \cite{liu2017calibrated,gillen2018online, gupta_kamble_2018}. Related to our work, Raghavan et al. \cite{RaghavanSlWoWu18} study externalities that arise in the exploration of classical stochastic bandit algorithms when applied across different subpopulations. Finally, Bechavod et al. \cite{bechavod2019equal} study a similar notion of one-sided feedback in online learning with fairness in mind. Unlike us, the latter work does not apply to overlapping populations but instead they take a contextual bandit approach, focus on stochastic and not adversarial losses, and aim for balance notions of fairness (where, as we discussed, there is conflict between incentives of designer and the groups).

Our work applies adversarial online learning which was initiated by the seminal works of Littlestone and Warmuth \cite{Littlestone:1994:WMA:184036.184040}, and Freund and Schapire \cite{FreundSc1997}. The classical experts guarantee we use in our first reduction can come from any of these or later developed algorithms. The apple tasting setting we consider to model the one-sided feedback was introduced by Helmbold et al. \cite{HelmboldLL92_apples}; a related concept is that of label-efficient prediction that instead has an upper bound on the number of times the learner can explore \cite{cesa2005minimizing}. Sleeping experts have been introduced by Blum \cite{blum1997empirical} and Freund et al. \cite{freund1997using}. Subsequently they were extended to a more general case of confidence-rated experts, and the results were better optimized \cite{blum2007external, gaillard2014second, DBLP:conf/colt/LuoS15}. The sleeping expert full feedback guarantee we use in our reductions is the one of AdaNormalHedge \cite{DBLP:conf/colt/LuoS15}. To the best of our knowledge, our work is the first to consider sleeping experts in the context of either apple tasting or label efficient prediction. We do so to enable our algorithms to deal with overlapping populations while only using realistic feedback (incorporating the one-sided nature of feedback).

%\section{Sleeping experts towards subgroup regret guarantees}

\section{Basic model}
\label{ssec:model}
% !TEX root = main.tex
\paragraph{Online learning with group context.} We consider an online learning setting with multiple overlapping demographic groups.  The set of groups $\mathcal{G}$ can correspond to divisions based on gender, age, ethnicity, or other attributes. People (also referred to as examples) arrive sequentially, and the example at round $t=1,2,\dots,T$ can be simultaneously a member of multiple groups (e.g., \emph{female} and \emph{hispanic}); the subset $\mathcal{G}_t\subseteq\mathcal{G}$ denotes all groups that person $t$ belongs to. 

At each round $t$, the system designer (or \emph{learning algorithm} or \emph{learner}) denoted by $\mathcal{A}$ aims to classify incoming examples by predicting a label $\hat{y}_t\in \hat{\mathcal{Y}}$. For example, in binary classication with deterministic predictors, the prediction space consists of positive and negative labels, i.e. $\hat{\mathcal{Y}}\in\crl{+,-}$ (e.g. whether the corresponding person should be admitted to an honors class). To assist her goal, the designer has access to a set $\mathcal{F}_t$ of rules that suggest particular labels according to the features of the example; these are typically referred to as \emph{experts} although they are not necessarily associated to any particular external knowledge. At every example, each expert $f\in\mathcal{F}_t$ makes a prediction $\hat{y}_{t,f}\in\hat{\mathcal{Y}}$ and the learner selects which expert's advice to follow in a (possibly randomized) manner; we use $p_{t,f}$ to denote the probability with which the learner follows the advice of expert $f$. Subsequently, the true label $y_t\in\mathcal{Y}$ is realized and each expert $f$ is associated with a loss $\ell_{t,f}\in[0,1]$. For example, if both the prediction and true label spaces are deterministic, i.e. $\hat{\mathcal{Y}}=\mathcal{Y}=\crl{+,-}$, then a reasonable loss is whether the prediction was correct: $\ell_{t,f}=\mathbf{1}\crl{\hat{y}_{t,f}\neq y_t}$. In order to not impose any i.i.d. assumption, we allow the losses to be adversarially selected. The learner then suffers expected loss $\hat{\ell}_t\prn*{\mathcal{A}}=\sum_{f\in\mathcal{F}_t} p_{t,f} \ell_{t,f}$ and observes some feedback (discussed below).

\paragraph{Feedback observed} In the first portion of this paper, we assume the learner receives \emph{full feedback}, i.e. at the end of the round she observes the losses of all experts (this typically can be achieved by observing the label). In Section~\ref{sec:sleeping_apples}, we turn our attention to the apple tasting setting in which we only receive feedback about the losses when we select the positive outcome. To ease notation and streamline presentation, we present here the performance and fairness notions for full feedback and defer the apple tasting definitions to Section~\ref{sec:sleeping_apples}.

\paragraph{Overall regret guarantee.} One natural goal for the designer in this setting is that it performs as well as the best among a class of experts $\mathcal{F}$ that are available every round, i.e.  $\forall t: \mathcal{F}\subseteq\mathcal{F}_t$. Intuitively, when there exists a rule that consistently classifies examples more accurately, the learner should eventually realize this fact and trust it more often (or at least perform as well via combining other rules). This is formalized by the following notion of regret.
\begin{equation}
    \regret=\sum_{t=1}^T \sum_{f\in\mathcal{F}_t}p_{t,f}\ell_{t,f}-\min_{f^{\star}\in\mathcal{F}}\sum_{t=1}^T\ell_{t,f^{\star}}.
\end{equation}
In the classical expert setting, $\mathcal{F}_t=\mathcal{F}$ for all rounds $t$. Algorithms such as the celebrated multiplicative weights \cite{FreundSc1997} incur regret that, on average, vanishes over time at a rate of $\sqrt{\log(\abs{\mathcal{F}})/T}$. Hence, despite the input not being i.i.d., the penalty we pay for not knowing in advance which expert is the best is very small and goes to $0$ at a fast rate. We allow changes in the sets $\mathcal{F}_t$ to incorporate adaptive addition of rules as well as group-specific rules.

\paragraph{Subgroup regret guarantees.} 
In order to not treat any group worse than what is inevitable, we are especially interested in group-based performance guarantees. In particular, for each group $g\in\mathcal{G}$, we want the performance on its members to be nearly as good as the best expert in a class $\mathcal{F}(g)$. This class can consist of all rules in $\mathcal{F}$, in which case this means we care not only about competing with the best rule in the class overall but also having the same guarantee for each group. It also allows each group to have rules specialized to it; for example, a third-party entity may observe a disparity in the performance of the group compared to what is achievable and recommend the use of a particular rule. To ease presentation we assume that the set $\mathcal{F}(g)$ is fixed in advance but most of our results extend to the case where new rules are added adaptively as potential unfairness is observed (see Remark~\ref{rem:adaptively_adding_experts}). This notion of group-based regret is formalized below: 
\begin{equation}\label{eq:subgroup_regret}
    \regret(g)=\sum_{t: g\in\mathcal{G}_t} \sum_{f\in\mathcal{F}_t}p_{t,f}\ell_{t,f}-\min_{f^{\star}\in\mathcal{F}(g)}\sum_{t:g\in\mathcal{G}_t}\ell_{t,f^{\star}}.
\end{equation}
In Section~\ref{ssec:subgroup_fairness_sleeping_experts}, we show that, via connecting to the literature of sleeping experts, we can have the average group-based regret vanish across time for all groups while still retaining the overall regret guarantee described above. In Section~\ref{sec:incompatibility_example}, we show that this is not in general possible for objectives based on fixed averages of false-positive and false-negative rates.

\section{Sleeping experts and one-sided feedback}
\label{ssec:subgroup_fairness_sleeping_experts}
% !TEX root = main.tex
In this section, we focus on subgroup regret guarantees and conceptually connect the quest for these guarantees to the literature of sleeping experts and the incentives of the groups.

\subsection{Subgroup regret guarantees via sleeping experts}

\paragraph{Sleeping experts.} Sleeping expert algorithms were originally developed to seamlessly combine task-specific rules so that their coexistence does not create negative externalities to other tasks. More formally, there is a set of experts $\mathcal{H}$ and, at each round $t$, any expert $h\in\mathcal{H}$ may decide to either become active (fire) or abstain (sleep); the set of experts that fire at round $t$ is denoted by $\mathcal{H}_t$. At any round, the algorithm can only select among active experts, i.e. puts non-zero probability $p_{t,h}$ only on experts $h\in \mathcal{H}_t$. The goal is that, for all experts $h^{\star}\in\mathcal{H}$, the performance of the algorithm in the rounds where the experts fire is at least as good as the one of $h^{\star}$. More formally the sleeping regret is:
\begin{equation}\label{eq:sleeping_regret}
    \sleepingregret(h^{\star})=\sum_{t: h^{\star}\in\mathcal{H}_t} \sum_{h\in\mathcal{H}}p_{t,h}\ell_{t,h}-\sum_{t:h^{\star}\in\mathcal{H}_t}\ell_{t,h^{\star}}.
\end{equation}
The goal is to ensure that the average sleeping regret for any sleeping expert $h\in\mathcal{H}$ is vanishing with the number of times the expert fires, $T(h)=\abs{t:h\in\mathcal{H}(t)}$, Multiple algorithms \cite{blum1997empirical,freund1997using,blum2007external, gaillard2014second,DBLP:conf/colt/LuoS15} achieve this goal. Probably the most effective among them is AdaNormalHedge by Luo and Schapire~\cite{DBLP:conf/colt/LuoS15} with an average sleeping regret of $O\prn*{\sqrt{\log(\abs{H})/T(h)}}$. We elaborate on this algorithm in Section~\ref{ssec:negative_IC} in order to discuss its game-theoretic properties.

\paragraph{Subgroup regret formulated via sleeping experts.} Looking closely at the definition of the desired subgroup regret from Eq. \eqref{eq:subgroup_regret} and the one of sleeping regret from Eq. \eqref{eq:sleeping_regret}, there are clear similarities, which motivates formulating our problem as a sleeping experts problem and applying algorithms such as AdaNormalHedge. This leads to the following theorem.
\begin{theorem}\label{thm:subgroup_via_sleeping}
\vspace{0.1in}
Let $\mathcal{A}$ be an algorithm with sleeping regret bounded by $\bigO\prn*{\sqrt{T(h)\cdot \log(\abs{\mathcal{H}}})}$ for any expert $h\in\mathcal{H}$ where $\mathcal{H}$ is any class. Let $\mathcal{G}$ be a set of overlapping groups and $N=\abs{\mathcal{F}}+\sum_{g\in\mathcal{G}} \abs{\mathcal{F}(g)}$. Then $\mathcal{A}$ can provide subgroup regret guarantee of $\bigO\prn*{\sqrt{T(g)\cdot \log(N)}}$ for any $g\in\mathcal{G}$ while ensuring overall regret guarantee of $\bigO\prn{\sqrt{T \log(N)}}$.
\end{theorem}
\begin{proof}
The idea to bound subgroup regret as a sleeping experts problem is to have each expert $f\in\mathcal{F}(g)$ fire only for members of group $g$, guaranteeing that they experience performance at least as good as its own. One small issue is that we may want to use the same rule $f\in\mathcal{F}_t$ for multiple subgroups; to deal with that, we create different copies of this expert, each associated to one group. 

More formally, we create a set of global sleeping experts $\mathcal{H}$ with one sleeping expert $h\in\mathcal{H}$ for every expert $f\in\mathcal{F}$. These sleeping experts fire every round and ensure the overall regret guarantee. Subsequently we create disjoint sets $\mathcal{H}(g)$ for each group $g\in\mathcal{G}$ where again we create a sleeping expert $h\in\mathcal{H}(g)$ for any expert $f\in\mathcal{F}(g)$. These sleeping experts fire only when the example is a member of group $g$ and hence ensure the subgroup regret guarantee. The eventual sleeping expert set is $\bigcup_{g\in \mathcal{G}} \mathcal{H}(g) \cup \mathcal{H}$. The cardinality of this set is $N=\abs{\mathcal{F}}+\sum_{g\in\mathcal{G}}\abs{\mathcal{F}(g)}$.
This formulation enables to automatically apply sleeping experts algorithms achieving subgroup regret of $\sqrt{T(g)\cdot\log(N})$ for any group $g\in\mathcal{G}$ while also guaranteeing an overall regret of $\sqrt{T\cdot\log(N)}$.
\end{proof}

\begin{remark}\label{rem:adaptively_adding_experts}
\vspace{0.1in} 
In the above formulation, we assumed that all the experts exist in the beginning of the time-horizon. However, the sleeping experts algorithms allow for experts to be added dynamically over time (treating them as not firing in the initial rounds). Hence we can adaptively add new sleeping experts if a group or some third-party entity suggests it, guaranteeing that we do at least as well in the remainder of the game on the members of the group without affecting with the subgroup regret guarantees of other overlapping groups. 
\end{remark}

\subsection{Subgroup regret with one-sided feedback}
\label{sec:sleeping_apples}
% !TEX root = subgroup_fairness_sleeping_experts.tex

We now move our attention to the more realistic setting where we receive the label of the example (and therefore learn the losses of the experts) only if we select the positive outcome (e.g. admit the student to the honors class and observe her performance). This is captured by the so-called \emph{apple tasting} setting which dates back to the work of Helmbold et al. \cite{HelmboldLL92_apples}.

\paragraph{Pay-for-feedback.} Our algorithms opearate in a more challenging model where, at the beginning of each round, the learner needs to select whether to ask for the label. If she asks for it, she receives it at a cost of $1$ instead of the loss of this outcome. Any guarantee for this setting is automatically an upper bound for the apple tasting setting. We can transform any such algorithm to an apple tasting one by selecting the positive outcomes at the rounds that we ask the label and ignoring any extra feedback. The loss of the positive outcome is upper bounded by $1$; since we assume the losses to be bounded in $[0,1]$, the loss in the pay-per feedback model is therefore only larger.

There are a few reasons why we want to work on this more stringent setting instead of the classical apple tasting setting. First, this feedback model makes it easy to create an estimator that is unbiased (since it does not condition on the prediction for the example and therefore our estimates do not suffer from selection bias). Second, in some applications, this model actually is more appropriate; for example, one may need to poll the participant to learn about their experience (which may be independent of whether they were classified as positive). Finally, this serves as an upper bound on the apple tasting setting; as we will see, arguing about lower bounds in this setting is significantly easier. Using the feedback in a more fine-grained way is an interesting open direction.

We now offer three guarantees for sleeping experts in the pay-per feedback setting, which are all achieved via black-box reductions to full-feedback algorithms (either sleeping experts or classical experts). Although the reductions become more involved as we proceed in the paper, no guarantee strictly dominates each other. Our algorithms select random points of exploration (when they ask for the labels to receive feedback). We denote by $\mathcal{E}$ the set of rounds that the algorithm ended up exploring (this is a random variable). For ease of presentation, we assume that we know the size of each demographic group and of all the intersections among groups; if these are not known, we can apply the so called doubling trick (similarly to the way described in Section~\ref{ssec:positive_IC}).

\paragraph{First reduction: Independent classical experts algorithms per intersection.}
The first reduction we provide comes from treating all disjoint intersections between subgroups separately and running separate apple-tasting versions of classical (non-sleeping) experts algorithms on each intersection. Although this provides optimal dependence on the size of each subpopulation, the guarantee has an exponential dependence on the number of different groups. 

For each disjoint intersection $I$ between groups, let $T(I)$ be the size of this intersection and denote by $g\in I$ the case when intersection $I$ includes $g$. Our algorithm splits the examples that lie in this intersection in $(T(I))^{2/3}$ phases, each of which consists of $(T(I))^{1/3}$ examples. At every phase, we select one random point of exploration. Whenever an example comes that belongs in $I$, our algorithm follows the advice of a classical experts algorithm (e.g. multiplicative weights) that is associated to $I$. This experts algorithm is updated at the end of the phase by the sample of the exploration round. This construction is in the spirit of Awerbuch and Mansour~\cite{AwerbuchMansour2003adapting}.

\begin{theorem}\label{thm:first_reduction}
\vspace{0.1in}
Let $\mathcal{A}$ be an algorithm with regret bounded by $\bigO{\sqrt{T\cdot\log(\mathcal{\abs{H}}})}$ when compared to an expert class $\mathcal{H}$, run on $T$ examples and split the examples in disjoint intersections, where each intersection corresponds to a distinct profile of subgroup memberships. For each intersection $I$, randomly selecting an exploration point every $(T(I))^{1/3}$ examples and running separate versions of $\mathcal{A}$ for each $I$ provides subgroup regret on group $g\in\mathcal{G}$ of $$\bigO{\prn*{\prn*{{2^{|\mathcal{G}|}}}^{1/3}\cdot \prn*{T(g)}^{2/3}\cdot \sqrt{\log\prn*{N}}}}$$
where $T(g)=|t:g\in\mathcal{G}(t)|$ is the size of the $g$-population and $N=\abs{\mathcal{F}}+\sum_{g\in\mathcal{G}} \abs{\mathcal{F}(g)}$.
\end{theorem}
\begin{proof}[Proof sketch.]
The guarantee follows from three observations formalized in Appendix~\ref{app:first_reduction}.
\begin{enumerate}
\item Among the exploration points, we run a classical experts algorithm so, on these examples, we have a regret guarantee that is square-root of the number of these examples.
\item For each phase, the exploration point is randomly selected and therefore the regret that we incur in the exploration point is an unbiased estimator of the average regret we incur in the whole phase (since the distribution of the algorithm in the phase is the same). As a result, the total regret in a phase is in expectation $(T(I))^{1/3}$ times the regret at the exploration point. 
\item A particular group can have examples in at most $2^{|G|}$ intersections (as this is all the possible membership relationships with respect to the demographic groups). 
\end{enumerate}
\end{proof}

\paragraph{Second reduction: Sleeping experts with fixed exploration phases.}
Aiming to avoid the exponential dependence on the number of groups, we now apply a sleeping experts algorithm such as AdaNormalHedge as our base algorithm. The algorithm described in this part removes this exponential dependence but introduces a dependence on the time horizon and therefore the regret guarantee can be suboptimal for minority populations whose size is significantly smaller than the total population. On the other hand, when all populations are well represented (and are of the same order as the time-horizon) then the guarantee has the optimal dependence on the size of the population without suffering in the number of groups.

The algorithm splits the examples in $T^{2/3}$ phases and selects one random point in each of the phases. Each phase consists in total of $T^{1/3}$ examples but its examples can be distributed differently across it. At the end of the phase, we update a sleeping experts algorithm (e.g. AdaNormalHedge) based on the observations at the exploration point.

\begin{theorem}\label{thm:second_reduction}
\vspace{0.1in}
Let $\mathcal{A}$ be an algorithm with sleeping regret bounded by $\bigO\prn*{\sqrt{T(h)\cdot \log(\abs{\mathcal{H}}})}$
for any expert $h$ in class $\mathcal{H}$. Randomly selecting an exploration point every $T^{1/3}$ examples (irrespective of what groups they come from) and running $\mathcal{A}$ on these points provides subgroup regret on group $g\in\mathcal{G}$ of
$$
\bigO\prn*{T^{1/6}\cdot (T(g))^{1/2}\cdot \sqrt{\log(N)}}
$$
where $T(g)=|t:g\in \mathcal{G}(t)|$ is the size of the group-$g$ population and $N=\abs{\mathcal{F}}+\sum_{g\in\mathcal{G}} \abs{\mathcal{F}(g)}$.
\end{theorem}
\begin{proof}[Proof sketch.]
The guarantee follows from two observations formalized in Appendix~\ref{app:second_reduction}.
\begin{enumerate}
\item Given that we run a sleeping experts algorithm across the exploration points, if we just focused on those examples, we simultaneously satisfy regret on them that is square-root of their size.
\item Within any phase the exploration point is selected uniformly at random. As a result, it is an unbiased estimator of the average regret we incur in the whole phase. Note that this is now the average across all rounds and not only rounds where we have members of the particular group which results in the dependence on the time-horizon $T$. 
\end{enumerate}
\end{proof}

\paragraph{Third reduction: Sleeping experts with adaptive exploration phases}
Our final reduction aims to remove the dependence on the time horizon while also avoiding an exponential dependence on the number of groups. Towards this end, we make the size of the phases adaptive based on the sizes of the populations. Our guarantee has both the aforementioned desired properties. On the negative side, the exponent on the group size is suboptimal (see discussion in the end of the section).

We again use a sleeping experts algorithm $\mathcal{A}$ across phases but phases are now designed adaptively. At the beginning of each phase $r$, we initialize a counter per group to capture the number of examples we have seen from it in phase $r$. When an example arrives, we increase the corresponding counters for all groups related to the example (possibly the example belongs in multiple groups; then we increase all the corresponding counters simultaneously). The phase ends when one of the groups $g\in\mathcal{G}$ has received $(T(g))^{1/4}$ examples in this phase. 

At the beginning of a phase $r$, we draw for each group $g\in\mathcal{G}$ a uniform random number $X(g,r)\in\crl{1,2,\dots,T(g)}$. This determines the exploration round for group $g$ at phase $r$; let $t(r,g)$ be the random variable determining the time that the $X(g,r)$-th example (after time $\tau_r$) from group $g$ will arrive. If the phase ends before this example arrives, i.e. $t(r,g)>\tau_r$ then we associate this phase with $0$ estimated losses for group $g$: $\tilde{\ell}_f(r,g)=0$ for all $f\in\mathcal{F}(g)$. Otherwise, the estimated loss corresponding to the phase is the loss at the exploration point, i.e. $\tilde{\ell}_f(r,g)=\ell_{t(r,g),f}$. Since the phase ends once any group counter reaches the upper bound on examples for its phase, this means that we may not reach the exploration point for some other groups (this is the reason why we may have $t(r,g)>\tau_{r+1}$\tledit{)}. Before proceeding to the next phase, we feed the estimated losses $\tilde{\ell}(r,g)$ to $\mathcal{A}$.

\begin{theorem}\label{thm:third_reduction}
\vspace{0.1in}
Applying the above algorithm provides subgroup regret on group $g\in\mathcal{G}$ of
$$
\bigO\prn*{|\mathcal{G}|\cdot (T(g))^{3/4}\cdot\sqrt{\log(N)}}
$$
where $T(g)=|t:g\in\mathcal{G}(t)|$ is the size of the $g$-population and $N=\abs{\mathcal{F}}+\sum_{g\in\mathcal{G}} \abs{\mathcal{F}(g)}$.
\end{theorem}
\begin{proof}[Proof sketch]
There are three components to prove this guarantee, formalized in Appendix~\ref{app:third_reduction}.
\begin{enumerate}
    \item The number for relevant phases of each group (phases where they have at least one example) is at most $T(g)$. This provides an upper bound on the number of phases that we need to consider with respect to group $g$.
    \item Using a similar analysis as before, we can create a guarantee about the regret we are incurring in the exploration points and multiply it by the $(T(g))^{1/4}$ which is the size of the phase. This would have created a completely unbiased estimator if there was no overlap with other groups.
    \item A final complication is that exploration may occur due to other groups so we need to understand how much we lose there. For that, we observe that smaller groups are explored with higher probability (the interaction with larger groups does not therefore significantly increase their probability of inspection). On the other hand, larger groups do not often collide with significantly smaller groups due to the latters' size.
\end{enumerate}
\end{proof}

\paragraph{On the optimality of the bounds.} In the pay-for-feedback model, even if we do not work in a sleeping experts setting, the best that one can hope for is a guarantee of $T^{2/3}$. If we explore $T^{a}$ examples, we obtain a regret of $T^{a/2}$ on them. If the estimator is unbiased, we need to then multiply this with $T^{1-a}$, which gives a regret of $T^{1-a/2}$. Since we pay for feedback, we also lose $T^{a}$. The maximum of the two is minimized when $a=2$. As a result, even if the groups were disjoint, we cannot hope for a better subgroup regret than $(T(g))^{2/3}$. Note that our results do not quite achieve this bound: having either a multiplicative term exponential in the number of groups (Theorem~\ref{thm:first_reduction}) or having a portion of the regret bound be in terms of the total time $T$ rather than $T(g)$ (Theorem~\ref{thm:second_reduction}). Achieving a bound of $(T(g))^{2/3}$ without enumerating over all possible disjoint intersections across groups is therefore an interesting open direction.

\section{A game-theoretic interpretation}
\label{sec:incentives}
% !TEX root = main.tex
In this section, we provide a game-theoretic interpretation of the above subgroup regret guarantee, connecting it to the notions of Individually Rationality (IR) and Incentive Compatibility (IC).

\paragraph{Individual Rationality.} In game theory, a mechanism is considered Individually Rational when the participants prefer to stay and be served in the system rather than to leave and use their best outside option. Consider each subgroup as a player in a game and the cost experienced by this player as the total loss in all its members; e.g., imagine each group has a representative who looks out for their best interests. This representative has access to the rules in $\mathcal{F}(g)$ (private type) and can opt to defect from the global learning system and create its own predictor with only rules in $\mathcal{F}(g)$. 

We say that a learning algorithm induces an IR mechanism if no group has significant incentive to opt out. (We cannot require zero incentive since the system needs some time to learn.)

The subgroup regret guarantee can be thought as an asymptotic version of the individual rationality property.  The guarantee ensures that the average benefit from being served outside of the system vanishes as the group size grows as formalized below.
\begin{definition}
\vspace{0.1in}
A learning algorithm induces an asympotically individual rational (IR) mechanism if no group gains (asymptotically) by getting served outside of the system.
$$
\forall g\in\mathcal{G}: 
\sum_{t:g\in\mathcal{G}_t} \sum_{f\in\mathcal{F}_t} p_{t,f}\ell_{t,f} - \min_{f^{\star}\in\mathcal{F}(g)}\sum_{t:g\in\mathcal{G}_t}\ell_{t,f^{\star}}=o\prn*{T(g)}.
$$
\end{definition}

\paragraph{Incentive Compatibility.} A second desired game-theoretic notion is that of Incentive Compatibility which states that the player has no incentive to misreport her true type. In our context, we define the type of a group to be the set $\mathcal{F}(g)$ of experts that it knows about, and IC means that group $g$ could not achieve enhanced performance for the group by hiding a subset of $\mathcal{F}(g)$ and removing it from the global learning process. Recall that $\hat{\ell}_t(\mathcal{A})=\sum_{f\in\mathcal{F}(t)}p_f^t\ell_f^t$ denote the expected loss of the algorithm $\mathcal{A}$ at round $t$. 

We say that a learning algorithm induces an IC mechanism if removing a subset of group-based experts does not improve the performance of the group. More formally, let $\mathcal{A}(\emptyset)$ denote the algorithm running with all the experts and nothing removed and $\mathcal{A}(\crl{g,H})$ denote the algorithm running with the subset $H\in\mathcal{F}(g)$ removed from the group-based experts and also from the overall set $\mathcal{F}$. Then:

\begin{definition}
\vspace{0.1in}
A learning algorithm $\mathcal{A}$ induces an asympotically incentive compatible (IC) mechanism if no group gains (asymptotically) by hiding a subset of its experts.
$$
\forall g\in\mathcal{G}, \forall H\in\mathcal{F}(g): 
\sum_{t:g\in\mathcal{G}(t)}\hat{\ell}_t\prn*{\mathcal{A}(\emptyset)}- \sum_{t:g\in\mathcal{G}(t)} \hat{\ell}_t\prn*{\mathcal{A}(\crl{g,H})}= o\prn*{T(g)}.
$$
\end{definition}

As before, we cannot hope to achieve the IC property exactly as extra experts will definitely delay the learning process, but we would like to satisfy an approximate version of this property where the average benefit from removing some experts vanishes as the group size grows. This is a desirable property as, when satisfied, it suggests that the groups should not overly think about potential adverse effects of suggesting particular predictors but instead provide all their proposed rules.

\subsection{Roadblocks in applying sleeping experts towards Incentive Compatibility}\label{ssec:negative_IC}
In this section, we show a strong negative result about classical sleeping experts algorithms with respect to the IC property. To make this formal, we present the result for the AdaNormalHedge algorithm of Luo and Schapire \cite{DBLP:conf/colt/LuoS15} but a similar intuition carries over to other sleeping experts algorithms (see Section~\ref{ssec:open_IC}).

\paragraph{AdaNormalHedge.} AdaNormalHedge \cite{DBLP:conf/colt/LuoS15} is an algorithm with strong adaptive regret guarantees. Its sleeping experts version starts with a set of experts with cardinality $N$ and a prior distribution $q$ that is typically initialized uniformly: $q_i=1/N$ for all $i\in\brk{N}$. Every expert keeps two quantities $R_{t,i}$ capturing the total regret it has experienced so far in the rounds that it fired and $C_{t,i}$ capturing the desired regret guarantee. These parameters determine the weight of its expert which is expresssed using a potential function $\Phi(R,C)=\exp(\frac{\max(0,R)^2}{3C})$ giving rise to a weight function: $$w(R,C)=\frac{1}{2}\prn*{\Phi(R+1,C+1}-\Phi(R-1,C+1).$$
More formally, both the expert quantities are initialized to $0$, i.e. $R_{0,i}=C_{0,i}=0$. At round $t=1\dots T$, a set $A(t)$ of experts is activated and the learner predicts with probability proportional to the weight of its firing expert: $p_{t,i}\propto q_i \cdot w(R_{t-1,i},C_{t-1,i})\cdot \mathbf{1}\crl*{i\in F(t)}$.\footnote{Luo and Schapire \cite{DBLP:conf/colt/LuoS15} predict arbitrarily if all weights are $0$; we commit on selecting uniformly at random then.} The adversary then reveals the loss vector $\ell_t$ and the learner suffers loss $\hat{\ell}_t=\sum_{i\in\brk{N}} p_{t,i}\ell_{t,i}$. This gives rise to an instantaneous regret for each firing expert: $r_{t,i}=\prn*{\hat{\ell_t}-\ell_{t,i}}\cdot \mathbf{1}\crl{i\in A(t)}$ which is used to update the expert parameters: $R_{t,i}=R_{t-1,i}+r_{t,i}$ and $C_{t,i}=C_{t-1,i}+\abs{r_{t,i}}$, before proceeding to the next round.

The regret of this algorithm with respect to an expert $i\in\brk{N}$ is roughly of order $\sqrt{C_{T,i} \log(N)}$ which in the sleeping experts version gives a sleeping regret $\sqrt{\abs{t:i\in A(t)}\log(N)}$. This is why using such an algorithm for subgroup regret guarantees provides a guarantee of $\sqrt{T(g)\log(N)}$, ignoring constants.

\paragraph{The IC lower bound.} Although the subgroup regret guarantee that sleeping experts provide makes them satisfy an asymptotic version of the IR property, we now show that this is not the case for the IC property; we illustrate this for AdaNormalHedge and further discuss it in Section~\ref{ssec:open_IC}.

\begin{theorem}\label{thm:negative_IC}
\vspace{0.1in}
AdaNormalHedge does not induce an asymptotically incentive compatible mechanism, i.e. there exists an instance where a group can asymptotically benefit from hiding one of its experts.
\end{theorem}
\begin{proof}
Consider a setting with two groups where the bigger group $B$ consists of the whole population whereas the smaller group $S$ corresponds to half of the population.  Every odd round an example from $S$ arrives and every even round an example from $B \setminus S$ arrives. The algorithm has access to one global expert $\mathcal{F}=\crl{f}$ as well as one group-specific expert per group: $\mathcal{F}(g)=\crl{f(g)}$ for $g\in\crl{B,S}$. 
\begin{itemize}
    \item Both group-specific experts have always loss $\ell_{t,f(g)}=0.2$ if $g\in \mathcal{G}(t)$.
    \item The global expert is really bad at predicting group $S$: $\ell_{t,f}=1$ if $g\in S$ but makes no mistakes on the remaining population: $\ell_{t,f}=0$ if $g\in B\setminus S$.
\end{itemize} 

The high-level idea on why IC does not hold is the following. Group $B$ prefers to use expert $f(S)$ on members of group $S$ and expert $f$ on $B\setminus S$; this is achieved if it hides expert $f(B)$. However, if it does not, AdaNormalHedge ends up using often expert $f(B)$ on $B \setminus S$ which leads to a much higher loss. Intuitively, since $f(B)$ and f make predictions on exactly the same set of examples and $f(B)$ has lower total loss, then $f(B)$ gets much higher weight than $f$ ---  the algorithm therefore uses $f(B)$ instead of f on $B\setminus S$. 

More formally, suppose first that group $B$ hides expert $f(B)$. The sleeping regret guarantee for expert $f(S)$ guarantees that
expert $f(S)$ is selected in all but a vanishingly small number of rounds. As a result, asymptotically, the loss accumulated from group $S$ is  $0.2\cdot (T/2)$. Given that $f(B)$ is not an option (as it is hidden), in members $B\setminus S$ the algorithm can only select expert $f$ incurring $0$ loss.

We now show that, if $f(B)$ is not hidden, it incurs more loss in members of $B\setminus S$ without gaining on $S$. In this case, we show below that expert $f$ is selected with probability $p_{t,f}\leq 1/2$ after a few initial rounds. This leads to an additional loss of at least $0.2\cdot (T/4)$ on examples of $B\setminus S$ since we select expert $f(B)$ which has higher loss than $f$ on these examples in at least half of the even rounds in expectation. As a result, by not hiding expert $f(B)$, the algorithm selects it often which leads in an additional loss that is linear in $T$ -- this directly implies that group $B$ is better off by hiding this expert and enhancing the overall performance.

What is left is to show why the probability of selecting expert $f$ is indeed $p_{f,t}\leq 1/2$ after a few inital steps.  The sleeping regret guarantee for expert $f(S)$ implies that the cumulative (across rounds) probability of selecting expert $f$ on examples in $S$ until round $t$ is at most $\sqrt{t}$ up to constants and log factors. As a result, the expected loss of the algorithm until round $t$ is at most $\hat{\ell}_t\leq 0.2\cdot t/2+\sqrt{t}$ and the total instantaneous regret of experts $f$ and $f(B)$ on odd rounds (i.e., members of group $S$) is $$\sum_{\substack{\tau\leq t \\\tau \textrm{ is odd}}}r_{\tau,f}\geq 0.8\cdot t/2-\sqrt{t} \qquad \textrm{and} \qquad \sum_{\substack{\tau\leq t \\\tau \textrm{ is odd}}}r_{\tau,f(B)}\leq \sqrt{t} \qquad \textrm{respectively.}$$

On even rounds (i.e., members of $B\setminus S$), the algorithm selects either $f$ or $f(B)$; therefore its loss is at most $0.2$. This means that the instantaneous regret on these rounds is:
$$\sum_{\substack{\tau\leq t \\\tau \textrm{ is even}}}r_{\tau,f}\geq -0.2\cdot t/2 \qquad \textrm{and} \qquad \sum_{\substack{\tau\leq t \\\tau \textrm{ is even}}}r_{\tau,f(B)}\leq 0.2\cdot t/2 \qquad \textrm{respectively.}$$
As a result, after a few initial rounds, the cumulative instantaneous regret of $f$ is consistently negative and also smaller than the one of $f(B)$. This means that, by the construction of the potential function, the weight of $f$ is $0$ as $\Phi(R_{f,t+1}+1,C+1)=\Phi(R_{f,t+1}-1,C+1)=1$. Since $\Phi$ is an increasing function on its first argument and the probability is proportional on the weight, this means that $f$ has the smallest probability across all other experts who fire and therefore has probability of being selected at most $1/2$ which is what we wanted to show. Note that, when $f(B)$ is hidden, the weight still becomes $0$ but now $f$ is the sole alternative and is therefore always selected in members of $B\setminus S$.
\end{proof}
The reason why this negative result arises is that, by hiding a subset of the experts, the group can potentially lead the algorithm to use different experts in different disjoint parts of the population. Sleeping experts algorithms penalize each expert for its overall performance at times when it fires and does not distinguish disjoint subpopulations where it performs much better.

\subsection{Incentive Compatibility in a computationally inefficient way}\label{ssec:positive_IC}
\paragraph{Multiplicative weights for each disjoint subgroup.} We now turn our attention to the use of separate algorithms for each disjoint intersection of groups. This suffers from an exponential dependence on the number of groups but satisfies both IR and IC. In some sense, it therefore serves as an existential proof that satisfying these quantities in an online manner is feasible and creates an intriguing open question of whether this can be achieved in a computationally efficient manner. 

We run separate multiplicative weights algorithms for each disjoint group intersection. More formally, for every $S\in 2^{\abs{\mathcal{G}}}$, we run sub-algorithm $\mathcal{A}(S)$ on examples $t$ where $g\in S$ if and only if $g\in\mathcal{G}_t$. We denote by $\mathcal{A}_t$ this sub-algorithm. We assume that experts are not added adaptively for this part (if a new expert appears then we reinitialize the algorithm).
The sub-algorithm $\mathcal{A}(S)$ has as experts all experts that fire at examples associated with $S$ so all members of $\mathcal{F}$ or $\mathcal{F}(g)$ for some $g\in S$. We let $N(\mathcal{A})$ denote the number of those experts for ub-algorithm $\mathcal{A}$.

For sub-algorithm $\mathcal{A}$, each expert $i$ is initialized with weight $w_{t,i}(\mathcal{A})=1$. The algorithm selects experts proportionally to the weights in the corresponding sub-algorithm, i.e. $p_{t,i}=\frac{w_{t,i}(\mathcal{A}_t)}{\sum_{j\in\mathcal{F}(S)}w_{t,j}(\mathcal{A}_t)}$. We then multiplicatively update weights with learning rate $\eta$ only for $\mathcal{A}_t$: $w_{t+1,i}(\mathcal{A})=w_{t,i}\cdot (1-\eta)^{\ell_{t,i}\cdot \mathbf{1}\crl{\mathcal{A}=\mathcal{A}_t}}$. Denoting by $T(\mathcal{A})$ the size of the disjoint subgroup, we should let $\eta=\sqrt{\log(N(\mathcal{A}))/T(\mathcal{A})}$ for a guarantee sublinear to this size. To deal with the fact that we do not know the size of each disjoint subgroup in advance we apply the so called doubling trick, assuming that it is $2^r$ (initial $r=2$) and reinitializing the algorithm by increasing $r$ -- this can happen at most $\log(T)$ times.

This algorithm satisfies the vanishing subgroup regret property (thus also the asymptotic IR property) as any group consists of multiple disjoint subgroups and the group has vanishing regret within each of them; other than the regret term, it cannot hope to do better if it is served outside of the system with its own functions. We now show that these separate multiplicative weights algorithms also achieve the asymptotic IC property. For that, we use a nice property of multiplicative weights establishing that multiplicative weights not only does not do worse than the best expert, but it also does not do better. This was first formalized by Gofer and Mansour \cite{Gofer2016lower} and was used in a fairness context by Blum et al. \cite{BlumGuLySr18} to establish that equality of average losses across different groups is preserved when combining experts that have equally good performance across these groups.
\begin{theorem}\label{thm:positive_IC}
\vspace{0.1in}
Running separate multiplicative weights algorithms for each disjoint intersection among groups induces an asymptotically IC mechanism, i.e. no group can asymptotically benefit by hiding any of its experts.
\end{theorem}
\begin{proof}
One essential step of the proof is the property that multiplicative weights with a fixed learning rate has performance almost equal to the one of the best expert. This is exactly shown in the proof of Theorem 3 in \cite{BlumGuLySr18} which establishes that for any sub-algorithm and any fixed learning rate, the performance of the sub-algorithm is equal to the performance of the best expert $L^{\star}$ in these rounds plus, up to constants, $\eta\cdot L^{\star}+\log(N(\mathcal{A}))/\eta$. Since we run each sub-algorithm with a fixed learning rate for at most $2^r$ rounds, this is, up to constants less than $\sqrt{2^{r}\log(N(\mathcal{A}))}$. Summing across all $r=2\dots T(\mathcal{A})$, this is at most $\sqrt{T(\mathcal{A})\log(N(\mathcal{A}))}$.

To now establish the asymptotic IC property, note that the intervals that have fixed learning rates are not affected by any groups' decision to hide a subset of their predictors. As a result, if some predictor is removed, then multiplicative weights will then have performance (asymptotically) close to the one of the best expert without the hidden one; this performance can only be worse (up to the regret term). This holds for any disjoint subgroup; it therefore still holds summing across all the (possibly exponential) disjoint subgroups.
\end{proof}

\subsection{Open problem: Computationally efficient Incentive Compatibility}\label{ssec:open_IC}
In Section~\ref{ssec:negative_IC}, we showed that AdaNormalHedge does not induce an IC mechanism as it used a single weight that deteriorated significantly in some disjoing subgroups and made some experts not usable in places that were essential and where hiding some experts would help with that. Using a single weight for each sleeping expert and updating it based on its instantaneous regret is not unique to AdaNormalHedge but is present to other algorithms such as the one of Blum and Mansour~\cite{blum2007external}. Since the algorithms do not distinguish between instantaneous regret obtained by different disjoint subgroups, the same construction can extend to those as well. On the other hand, in Section~\ref{ssec:positive_IC}, we also showed that separate multiplicative weights for each subgroup provide the asymptotic IC property as they have the nice property that the performance of the algorithm is exactly as good as the one of the best expert in this group. On the negative side, having one sub-algorithm per disjoint subgroup is computationally costly as it means that we have exponential number of sub-algorithms.

These two results lead to the following open question: can we design a computationally efficient algorithm (not enumerating over all disjoint subgroups) that satisfies the subgroup regret guarantee? Besides its application to subgroup fairness, this question has independent technical interest as it seems to require a fundamentally new approach for the sleeping experts problem.

\section{Impossibility results for fixed averages of FNR and FPR}
\label{sec:incompatibility_example}
% !TEX root = main.tex
In this section we demonstrate that if rather than minimizing the {\em fraction of errors}, each group wishes to minimize the unweighted average of its False Negative Rate (FNR) and False Positive Rate (FPR), or any fixed nontrivial weighted average of these two quantities, then guarantees of the form given in earlier sections are intrinsically not possible when there is no zero-error predictor.\footnote{Minimizing just FNR (resp.~just FPR) is trivial by predicting positive (resp.~negative) on every example.}

\paragraph{False negative and false positive rates.} 
Many fairness notions explicitly consider false negative and false positive rates.  For example, the notion of Equalized Odds \cite{hardt2016equality} imposes that both these rates should be the same for all groups of interest, and the notion of Equality of Opportunity \cite{hardt2016equality} imposes equality for just the FNR.  Given the attention paid to false negative and false positive rates, it is natural to consider the objective of minimizing their average (or minimizing their maximum, which our negative results will also apply to).  More formally, under this objective, the loss of the algorithm corresponds to $\frac{1}{|t:y_t=+|}\sum_{t:y_t=+}\hat{\ell}_t+\frac{1}{|t:y_t=-|}\sum_{t:y_t=-}\hat{\ell}_t$. If we have just one group and aim to achieve performance compared to the one of the best expert $f^{\star}\in \mathcal{F}(g)$ for this objective, i.e. $\frac{1}{|t:y_t=+|}\sum_{t:y_t=+}\ell_{t,f^{\star}}+\frac{1}{|t:y_t=-|}\sum_{t:y_t=-}\ell_{t,f^{\star}}$, this is feasible even in an online setting by running a no-regret algorithm and weighting each example based on the number of examples with the same label that exist. Also note that, unlike balance notions, {\em if groups are disjoint} then simultaneously optimizing this objective for both groups does not create some inherent conflict between groups, e.g. necessitating to have performance on some group that is worse than necessary.

Unfortunately, we show that while minimizing the unweighted average of FNR and FPR seems like a benign objective, it can be impossible to produce a global prediction strategy that, on each group, performs nearly as well as the best predictor given for that group under this objective. That is, even though all groups appear to have objectives that are ``aligned'' they may not be simultaneously satisfiable. This impossibility result holds due to purely statistical reasons even in the batch setting and has nothing to do with the online or non-i.i.d. nature of examples. Note that if the populations are not overlapping, the batch setting for this goal is straightforward: among the available classifiers, for each group select the one with the best performance according to the given objective. The following theorem shows that, with overlapping populations, this is no longer possible.

\begin{theorem}\label{thm:unweighted_average_example}
\vspace{0.1in}
Consider overlapping groups wishing to achieve performance on their members comparable to their group-specific predictors with respect to the objective of unweighted average of false positive and false negative rates. Even in the batch setting, there exist settings where it is impossible to simultaneously achieve this goal for both groups.
\end{theorem}
\begin{proof}
Assume two groups $\mathcal{G}=\crl{A,B}$ that have $80\%$ of their examples disjoint from the other group and the remainder $20\%$ in common, as illustrated in Figure \ref{fig:unweighted_average}. All the non-overlapping examples of group $A$ have positive label and the non-overlapping examples of group $B$ have negative label. With respect to the shared portion, the examples have labels that are uniformly distributed: conditioned on being in the intersection, the probability of the label being positive is 0.5.

Regarding the predictors, we have two predictors $\mathcal{F}=\crl{f_A,f_B}$. Predictor $f_A$ correctly predicts positive in the non-overlapping examples ($A \setminus B$), and predicts negative in the overlapping part $A \cap B$. The false positive rate of this predictor is $0$ and the false negative rate is $1/9$. Analogously, predictor $f_B$ predicts negative on the examples in $B \setminus A$ and predicts positive in $A\cap B$, resulting in a false positive rate on examples in $B$ of $1/9$ and a false negative rate of $0$. Therefore, for both $g\in \crl{A,B}$, $f_g$ has a generalization error of $1/18$ with respect to the unweighted average of false positive and false negative rates.

We now show that this performance cannot (even approximately) be simultaneously achieved for both groups by combining these two predictors even in the batch setting. In particular, since examples in $A\cap B$ all have uniformly random labels, we can only select some probability $p$ to predict a positive label for points in the intersection. Even if we perfectly classify all the examples not in the intersection, if $p>1/2$ then the false positive rate of $A$ is higher than $1/2$ (as it misclassifies half of the negative examples); otherwise the false negative rate of $B$ is no less than $1/2$.  This means that one of the two populations will necessarily have unweighted average of false positive and false negative rates higher than $1/4$ and therefore will not achieve performance even approximately as good as its best predictor.
\end{proof}

\begin{figure}[!h]
\centering
\includegraphics[width=0.6\textwidth]{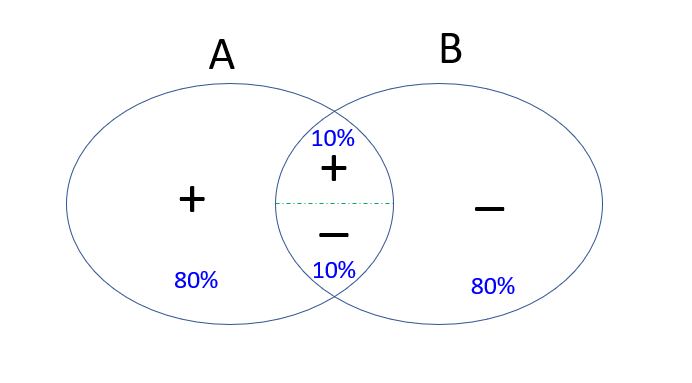}
\caption{Construction illustrating that unweighted average of false negative and false positive rate cannot be simultaneously optimized with respect to overlapping populations (Theorem \ref{thm:unweighted_average_example}).}
\label{fig:unweighted_average}
\end{figure}

\paragraph{Discussion.} The above result illustrates the additional complications caused by overlapping groups and highlights a subtlety in the positive results we provided in the previous sections. In particular, the key distinction is that for the objective of average loss (or error rate), the average contribution of any given example is the same across all groups; a misclassification is equally damaging for all groups an example belongs to. This is not the case in the setting considered here of minimizing the average (or maximum or any fixed nontrivial combination) of FPR and FNR where the harm of each misclassified example needs to be weighted differently across groups depending on their proportion of positive and negative examples.

\section{Conclusions}
\label{sec:discussion}
% !TEX root = main.tex
In this paper, we consider settings where different overlapping populations coexist and we wish to design algorithms that do not treat any population unfairly. We consider a notion of fairness that corresponds to predicting as well as humanly (or algorithmically) possible on each given group, rather than based on requirements for equality. This framework can directly incorporate a designer's goal of good overall prediction by creating one extra group (for the designer) that includes all the examples. Our results extend to the more realistic one-sided feedback (apple tasting) setting and have a nice game-theoretic interpretation. Our work makes a step towards identifying fairness notions that work well in online settings and are compatible with the existence of multiple parties, each with their own interests in mind. Our impossibility results with respect to the average (or maximum) of false negative and false positive rates demonstrate that satisfying the interests of different overlapping populations is quite subtle, further highlighting the positive results. 

Regarding incentives, we show how to efficiently achieve Individual Rationality and how to inefficiently achieve both Individual Rationality and Incentive Compatibility. Achieving IR and IC together in a computationally efficient way (without enumerating across all disjoint intersections of subgroups) is a very interesting question that seems to require novel learning-theoretic ideas as we discuss in Section~\ref{ssec:open_IC}. On the apple tasting front, we provide three distinct guarantees; the resulting algorithms are near-optimal but suffer from orthogonal shortcomings as we discuss in Section~\ref{sec:sleeping_apples}. Avoiding these and achieving an optimal guarantee for sleeping experts with apple tasting is an interesting open question.

\subsection*{Acknowledgements}
The authors would like to thank Suriya Gunasekar for various useful discussions during the initial stages of this work and Manish Raghavan for offering comments in a preliminary version of the work.

\bibliographystyle{alpha}
\bibliography{bibliog}

\appendix
\section{Proof of Theorem~\ref{thm:first_reduction}}
\label{app:first_reduction}
% !TEX root = main.tex
\paragraph{Theorem \ref{thm:first_reduction} restated.}
Let $\mathcal{A}$ be an algorithm with regret bounded by $\bigO{\sqrt{T\cdot\log(\mathcal{\abs{H}}})}$ when compared to an expert class $\mathcal{H}$, run on $T$ examples and split the examples in disjoint intersections, where each intersection corresponds to a distinct profile of subgroup memberships. For each intersection $I$, randomly selecting an exploration point every $(T(I))^{1/3}$ examples and running separate versions of $\mathcal{A}$ for each $I$ provides subgroup regret on group $g\in\mathcal{G}$ of $$\bigO{\prn*{\prn*{{2^{|\mathcal{G}|}}}^{1/3}\cdot \prn*{T(g)}^{2/3}\cdot \sqrt{\log\prn*{N}}}}$$
where $T(g)=|t:g\in\mathcal{G}(t)|$ is the size of the $g$-population and $N=\abs{\mathcal{F}}+\sum_{g\in\mathcal{G}} \abs{\mathcal{F}(g)}$.
\begin{proof}
The guarantee follows from three observations:
\begin{enumerate}
\item Among the exploration points, we run a classical experts algorithm so, on these examples, we have a regret guarantee that is square-root of the number of these examples.
\item For each phase, the exploration point is randomly selected and therefore the regret that we incur in the exploration point is an unbiased estimator of the average regret we incur in the whole phase (since the distribution of the algorithm in the phase is the same). As a result, the total regret in a phase is in expectation $(T(I))^{1/3}$ times the regret at the exploration point. 
\item A particular group can have examples in at $2^{|G|}$ intersections (as this is all the possible membership relationships with respect to the demographic groups). 
\end{enumerate}
We now formalize these three ideas, to obtain the guarantee. Let $\mathcal{F}(I)$ be the set of experts that are either global or belong to some $g\in I$, and update $\mathcal{F}(g)$ to include both the global experts and the group-specific experts. Initially we split the performance of our algorithm across all the intersections $I$ such that $g\in I$. Let $f^{\star}$ be the comparator with the minimum cumulative loss on $g$.
\begin{align*}
    \regret(g)=\sum_{t:g\in\mathcal{G}(t)}\sum_{f\in \mathcal{F}(g)}p_{t,f}\ell_{t,f}-\sum_{t:g\in\mathcal{G}(t)}\ell_{t,f^{\star}}=\sum_{I:g\in I}\sum_{t:g\in\mathcal{G}(t)\cap I}\prn*{\sum_{f\in \mathcal{F}(g)}p_{t,f}\ell_{t,f}-\ell_{t,f^{\star}}}
\end{align*}
Focusing on a particular intersection $I$ we connect its regret to the performance of the exploration points. Denote by $t(r, I)$ the exploration point for phase $r$ in intersection $I$. Also denote by $\tau(r)$ the beginning of the $r$-th phase. We use the fact that the exploration point is an unbiased representation on the regret at a phase (as it is selected uniformly and the algorithm is only updated at the end of the phase). Applying the guarantee for $\mathcal{A}$ on the exploration times of all phases, we obtain:
\begin{align*}
    \sum_{t:g\in\mathcal{G}(t)\cap I}\prn*{\sum_{f\in \mathcal{F}(g)}p_{t,f}\ell_{t,f}-\ell_{t,f^{\star}}}&=\sum_{r=1}^{(T(I))^{2/3}}\sum_{t=\tau_r}^{\tau_{r+1}-1}\prn*{\sum_{f\in \mathcal{F}(g)}p_{t,f}\ell_{t,f}-\ell_{t,f^{\star}}}\cdot \mathbf{1}\crl{g\in\mathcal{G}(t)\cap I}\\
    &= (T(I))^{1/3}\cdot \sum_{r=1}^{\prn*{T(I)}^{2/3}}\En\brk*{\sum_{f\in \mathcal{F}(g)}p_{t(r,I),f}\ell_{t(r,I),f}-\ell_{t(r,I),f^{\star}}}\\
    &\leq \prn*{T(I)}^{1/3}\sqrt{(T(I))^{2/3}\log\prn{|\mathcal{F}|}}=\prn*{T(I)}^{2/3}\cdot \sqrt{\log(|\mathcal{F}|)}
\end{align*}
where the expectations are taken over the random selections of $t(r,I)$. 

Combining the above and applying Holder inequality, we obtain:
\begin{align*}
    \reg(g)&=\sum_{I:g\in I}\sum_{t:g\in\mathcal{G}(t)\cap I}\prn*{\sum_{f\in \mathcal{F}(g)}p_{t,f}\ell_{t,f}-\ell_{t,f^{\star}}}\leq \sum_{I:g\in I} (T(I))^{2/3}\cdot \sqrt{\log(|\mathcal{F})}\\
    &\leq \prn*{\sum_{I:g\in I}1}^{1/3}\cdot \prn*{\sum_{I:g\in I}\prn*{T(I)}}^{2/3}\sqrt{\log(|\mathcal{F}|}\leq \prn*{2^{|\mathcal{G}|}}^{1/3}\cdot \prn*{T(g)}^{2/3}\cdot \sqrt{(\log(|\mathcal{F}|)}.
\end{align*}
Finally, we need to control how much we are losing via the exploration rounds. Since we apply pay-per-feedback we lose $1$ every time we inspect (an upper bound on the apple tasting loss). For any intersection, we have exactly $(T(I))^{2/3}$ such exploration points so we are not losing something more compared to what we previously discussed which completes the proof.
\end{proof}

\section{Proof of Theorem~\ref{thm:second_reduction}}
\label{app:second_reduction}
% !TEX root = main.tex
\paragraph{Theorem~\ref{thm:second_reduction} restated.}
Let $\mathcal{A}$ be an algorithm with sleeping regret bounded by $\bigO\prn*{\sqrt{T(h)\cdot \log(\abs{\mathcal{H}}})}$
for any expert $h\in\mathcal{H}$ where $\mathcal{H}$ is any class. Randomly selecting an exploration point every $T^{1/3}$ examples (irrespective of what groups they come from) and running $\mathcal{A}$ on these points provides subgroup regret on group $g\in\mathcal{G}$ of
$$
\bigO\prn*{T^{1/6}\cdot (T(g))^{-1/2}\cdot \sqrt{\log(N)}}
$$
where $T(g)=|t:g\in \mathcal{G}(t)|$ is the size of the $g$-population and $N=\abs{\mathcal{F}}+\sum_{g\in\mathcal{G}} \abs{\mathcal{F}(g)}$.
\begin{proof}
The guarantee follows from two observations:
\begin{enumerate}
\item Given that we run a sleeping experts algorithm across the exploration points, if we just focused on those examples, we simultaneously satisfy regret on them that is square-root of their size.
\item Within any phase the exploration point is uniformly at random selected. As a result, it is an unbiased estimator of the average regret we incur in the whole phase. Note that this is now the average across all rounds and not only rounds where we have members of the particular group which results to the dependence on the time-horizon $T$.  
\end{enumerate}
We now formalize these ideas. As before, we connect the regret on a group to the one of the exploration points. Denote by $t(r)$ the random variable that corresponds to the exploration point. Now the size of the phases is fixed in advance so $r$-th phase starts at $\tau_r=r\cdot T^{1/3}$. Similarly as before, we denote by $f^{\star}$ the comparator with the minimum cumulative loss on $g$. Applying linearity of expectation and Jensen's inequality, we obtain:
\begin{align*}
    \regret(g)&=\sum_{t:g\in\mathcal{G}(t)}\sum_{f\in\mathcal{F}(g)}p_{t,f}\ell_{t,f}-\sum_{t:g\in\mathcal{G}(t)}\ell_{t,f^{\star}}\\
    &=\sum_{r=1}^{T^{2/3}}\sum_{t=\tau_r}^{\tau_{r+1}-1}\prn*{\sum_{f\in\mathcal{F}(g)}p_{t,f}\ell_{t,f}-\ell_{t,f^{\star}}}\cdot \mathbf{1}\crl{g\in\mathcal{G}(t)}\\
    &=T^{1/3}\cdot  \sum_{r=1}^{T^{2/3}}
    \En\brk*{
    \prn*{\sum_{f\in\mathcal{F}(g)}p_{t(r),f}\ell_{t(r),f}-\ell_{t(r),f^{\star}}}\cdot \mathbf{1}\crl*{g\in\mathcal{G}\prn*{t(r)}}
    }\\
    &\leq T^{1/3}\cdot \En\brk*{\sqrt{\sum_{r=1}^{T^{2/3}}\mathbf{1}\crl{g\in \mathcal{G}(t(r))}\cdot\log\prn*{N}}}\\
    &\leq T^{1/3}\cdot \sqrt{\En\brk*{\sum_{r=1}^{T^{2/3}}\mathbf{1}\crl{g\in \mathcal{G}(t(r))}\cdot \log\prn{N}}}\\
    &= T^{1/3}\cdot \sqrt{T(g)\cdot T^{-1/3}\cdot \log\prn{N}}=T^{1/6}\cdot \sqrt{T(g)\cdot \log(N)}.
\end{align*}
The expectation is taken over the random selections of the exploration times $t(r)$. The second-to-last equality holds as each member of group $g$ has probability $T^{-1/3}$ to be an exploration point and therefore the expected number of exploration points in the group is $T(g)\cdot T^{-1/3}$. 

Finally, for the pay-for-feedback model, we need to also consider the effect of inspected rounds in loss (we are losing $1$ every time that we inspect). The number of these rounds on examples in $g$ is at most
$T(g)\cdot T^{-1/3}\leq (T(g))^{2/3}$ so lower order than the term we already have.
\end{proof}

\section{Proof of Theorem~\ref{thm:third_reduction}}
\label{app:third_reduction}
% !TEX root = main.tex
\paragraph{Theorem \ref{thm:third_reduction} restated.}
Applying the algorithm described in the third reduction (Section\ref{sec:sleeping_apples} provides subgroup regret on group $g\in\mathcal{G}$ of
$$
\bigO\prn*{|\mathcal{G}|\cdot (T(g))^{3/4}\cdot\sqrt{\log(N)}}
$$
where $T(g)=|t:g\in\mathcal{G}(t)|$ is the size of the $g$-population and $N=\abs{\mathcal{F}}+\sum_{g\in\mathcal{G}} \abs{\mathcal{F}(g)}$.
\begin{proof}
There are three important components to prove this guarantee.
\begin{enumerate}
    \item The number for relevant phases of each group (phases where they have at least one example) is at most $T(g)$. This provides an upper bound on the number of phases that we need to consider with respect to group $g$.
    \item Using a similar analysis as before, we can create a guarantee about the regret we are incurring in the exploration points and multiply it by the $(T(g))^{1/4}$ which is the size of the phase. This would have created a completely unbiased estimator if there was no overlap with other groups.
    \item A final complication is that exploration may occur due to other groups so we need to understand how much we lose there. For that, we observe that smaller groups are explored with higher probability (the interaction with larger groups does not therefore significantly increase their probability of inspection). On the other hand, larger groups do not often collide with significantly smaller groups due to the latters' size.
\end{enumerate}
More formally, to analyze the subgroup regret for $g\in\mathcal{G}$, let's consider a fictitious setting where all phases have equal size for group $g$. This can be done by padding $0$s in the end of the phase. This fictitious setting has the same loss as the original (as it only differs in that it has some more examples with zero loss). For this fictitious setting, the exploration point is an unbiased estimator of the average regret. We first analyze the subgroup regret assuming that the points of inspection in other groups do not overlap with $g$. Applying similar ideas as in the previous theorem:
\begin{align*}
    \regret(g)&=\sum_{t:g\in\mathcal{G}(t)}\sum_{f\in\mathcal{F}(g)}p_{t,f}\ell_{t,f}-\sum_{t:g\in\mathcal{G}(t)}\ell_{t,f^{\star}}\\
    &=\sum_{r=1}^{T(g)}\sum_{t=\tau_r}^{\tau_{r+1}-1}\prn*{\sum_{f\in\mathcal{F}(g)}p_{t,f}\ell_{t,f}-\ell_{t,f^{\star}}}\cdot \mathbf{1}\crl{g\in\mathcal{G}(t)}\\
        &=(T(g))^{1/4}\cdot  \sum_{r=1}^{T(g)}
    \En\brk*{
    \prn*{\sum_{f\in\mathcal{F}(g)}p_{t(r),f}\ell_{t(r,g),f}-\ell_{t(r,g),f^{\star}}}\cdot \mathbf{1}\crl*{g\in\mathcal{G}\prn*{t(r,g)}, t(r,g)\leq t_{r+1}}
    }\\
    &=(T(g))^{1/4}\cdot  \sum_{r=1}^{T(g)}
    \En\brk*{
    \prn*{\sum_{f\in\mathcal{F}(g)}p_{t(r),f}\tilde{\ell}_f(r,g)-\tilde{\ell}_{f^{\star}(r,g)}}
    }\\
    &\leq (T(g))^{1/4}\cdot \sqrt{T(g)\log(N)}=(T(g))^{3/4}\sqrt{\log(N)}
\end{align*}
The expectation is taken over the random selections of the exploration times $t(r,g)$. The last inequality holds as the number of non-zero entries in the previous sum is at most the number of phases and we run $\mathcal{A}$ on the estimated losses.

Let's understand now how much group $g$ is harmed by the fact that some of its examples may be inspected by overlapping groups. As a result, we need to count in the regret the contribution of these examples as well. Note that each group has at most one exploration point within a phase -- this is the reason why we get the dependence on $|\mathcal{G}|$ in our bound. 

Group $g$ is only affected by others' exploration if this happens on examples that are also members of group $g$. We first consider the effect of smaller groups than $g$. For any group $g'$ with $T(g')\leq T(g)$ the exploration points at group $g'$ are at most the size of the group times the probability that each example is an exploration point for $g'$, i.e. $T(g')\cdot (T(g'))^{-1/4}=(T(g'))^{3/4}$. Since the size of $g'$ is smaller than the size of $g$, this contributes at most an extra term of $(T(g))^{3/4}$ in the regret.

Let's now consider groups $g'$ with $T(g')> T(g)$. Then, even if all examples of $g$ are also examples of $g'$, the probability that each of them is an exploration point due to $g'$ is $(T(g'))^{-1/4}$. Hence the expected number of exploration points of $g'$ on examples in $g$ is at most $T(g)\cdot (T(g'))^{1/4}\leq (T(g))^{3/4}$.
\end{proof}

\end{document}